\documentclass[letterpaper, 10 pt, conference]{IEEEtran}

\IEEEoverridecommandlockouts                              

\usepackage{hyperref}
\usepackage[numbers]{natbib}

\usepackage{amsmath,amsthm,amssymb}

\usepackage{siunitx}
\usepackage{todonotes}
\usepackage{leftindex}
\usepackage{bm}
\usepackage{algorithmicx} 
\usepackage{algorithm}
\usepackage[noend]{algpseudocode}
\usepackage{caption}
\usepackage{subcaption}
\usepackage{graphicx}
\usepackage{gensymb} 
\usepackage{multirow}
\usepackage{booktabs}
\newif\ifproofread

\algrenewcommand\algorithmicrequire{\textbf{Input:}}
\algrenewcommand\algorithmicensure{\textbf{Output:}}

\usepackage[font=small,labelfont=bf]{caption}

\newtheorem{problem}{Problem}

\newtheorem{definition}{Definition}

\newtheorem{proposition}{Proposition}

\title{\LARGE \bf
GeoDEx: A Unified Geometric Framework for Tactile Dexterous and Extrinsic Manipulation under Force Uncertainty
}

\author{Sirui Chen$^{1,*}$, Sergio Aguilera Marinovic$^{2}$, Soshi Iba$^{2}$ and Rana Soltani Zarrin$^{2}$
\thanks{$^{1}$Computer Science Department, Stanford University, Stanford, CA 94305, USA.
        {\tt\small ericcsr@stanford.edu}
       }%
\thanks{$^{2}$Honda Research Institute USA, San Jose, CA 95134, USA. corresponding author:
       {\tt\small \ rana\_soltanizarrin@honda-ri.com}}%
\thanks{$^{*}$This work was conducted during an internship at Honda Research Institute USA.}        
}

\begin{document}

\proofreadfalse

\maketitle
\thispagestyle{empty}
\pagestyle{empty}

\begin{abstract}
Sense of touch that allows robots to detect contact and measure interaction forces enables them to perform challenging tasks such as grasping fragile objects or using tools. Tactile sensors in theory can equip the robots with such capabilities. However, accuracy of the measured forces is not on a par with those of the force sensors due to the potential calibration challenges and noise. 
This has limited the values these sensors can offer in manipulation applications that require force control.
In this paper, we introduce GeoDEx, a unified estimation, planning, and control framework using geometric primitives such as plane, cone and ellipsoid, which enables dexterous as well as extrinsic manipulation in the presence of uncertain force readings. Through various experimental results, we show that while relying on direct inaccurate and noisy force readings from tactile sensors results in unstable or failed manipulation, our method enables successful grasping and extrinsic manipulation of different objects. Additionally, compared to directly running optimization using SOCP (Second Order Cone Programming), planning and force estimation using our framework achieves a 14x speed-up.
\end{abstract}


\section{Introduction}
Dexterous manipulation capabilities are essential for robots to function effectively in human-centered environments, such as helping with tasks that involve handling different objects or tools~\cite{zarrin2023hybrid}. Planning and control methods for such tasks that can enable robust and generalizable contact-rich manipulation need contact force information. While force sensors can provide accurate force readings, physical limitations associated with embedding the sensors into the robotic hands, as well as lack of high-resolution tactile information limit the use of these sensors. On the other hand, recent developments in tactile sensors resulted in ever lighter and higher resolution sensors, which can be installed on different robot end-effectors such as fingertips of a dexterous hand. However, it is still quite challenging and expensive to develop tactile sensors that can provide accurate force readings, especially in both normal and shearing directions. While developing such tactile sensors is important, the ongoing research efforts in the field of dexterous manipulation require relying on the available tactile sensors with the aforementioned shortcomings. Previous works have shown that even using binary~\cite{yuan2024robot} or highly discretized~\cite{qi2023general} information from tactile sensors can already significantly improve performance during in-hand object reorientation. Thus, it would be valuable if we could exploit full tactile sensor readings to perform force control for dexterous manipulation.

In this paper, we propose GeoDEx, a unified framework that fully utilizes imperfect tactile sensor readings and can be used for force planning and control in dexterous grasping and extrinsic manipulation. GeoDEx formulates finding equilibrium forces on multi-point contact problems as optimization problems with different geometric primitives such as planes, cones, and ellipsoids. It is designed to handle sensor noise as well as lack of access to shearing force reading. Through various experiments, we have shown that GeoDEx enables admittance control based dexterous grasping of different objects with various numbers of fingers as well as extrinsic manipulation such as pivoting a cuboid and a screwdriver.

\section{Related work}
\subsection{State of Tactile Sensors}
Tactile sensor arrays have shown great potential in robot manipulation, especially in dexterous manipulation~\cite{tactilereview}. While some works focus on building tactile gloves to collect force data from humans to teach robots ~\cite{natureglove, baoeskin}, other works use tactile sensors to directly help robots to sense and control contact in applications such as dense packing~\cite{robopack}, grasping fragile objects~\cite{yan2021soft} and pouring water with bimanual dexterous hand~\cite{lin2024learning}. As one of the state-of-the-art tactile sensor arrays, Xela sensors~\cite{xela} is compact, deformable, and can provide 3D force direction and magnitude readings with contact locations. They have been applied to slip prediction and contact state classification~\cite{mandil2022action, zhou2020learning}. However, force readings from those sensors are inaccurate due to their sensitivity to magnetic interference and the gap between taxels~\cite{xela}. Another flexible tactile sensor array, Touchlab sensor~\cite{Touchlab} uses small piezo-electric sensor arrays to detect 1D force in the normal direction, which can provide more accurate force readings due to less interference. However, we found its force readings are still far less accurate compared to traditional force torque sensors and directly using its data to perform force control will often result in control failure.
\subsection{Utilizing Tactile Readings}
Given noisy and sparse tactile sensor readings, prior methods focus on estimating force and contact points from raw data using analytical models such as perturbation observer~\cite{perturbobs}, sensor deformation analysis~\cite{sensordeform}, and predictive models~\cite{Niederhauser2024}. For camera-based tactile sensors, \cite{visualdigit} has built depth maps from raw images and regressive analytical models to extract force data from depth maps. With the help of more accurate force-torque sensors to provide data labeling, learning-based methods have also been applied to map raw sensor data into force readings~\cite{nvidialearning, su2015force}. Compared to these methods, instead of directly working with raw sensor reading, our method serves as a post-processing stage after having these inaccurate force readings extracted from raw sensor data. Processed force readings may still be different from ground truth contact force value but can be directly used by downstream controllers to perform tasks such as grasp force control in dexterous grasping and force control in extrinsic manipulation. 

Sharing the same spirit, another branch of work focuses on extracting salient information from tactile sensors. \cite{yuan2024robot} uses inexpensive film pressure sensors to provide binary contact signals for helping dexterous manipulation. \cite{qi2023general} only extract 3-bit discretized force direction from fingertip tactile sensors for dexterous object in-hand reorientation. Other works also simplify tactile sensor readings to only distinguish sliding and slipping using tactile sensor signals~\cite{meier2016distinguishing} or detect binary contact information.\cite{veiga2018grip}. Our work not only can detect contact from tactile sensors, but it can also provide relative force magnitudes at different contact points, which can be used for the challenging tasks mentioned above.

\subsection{Tactile-Enabled Dexterous Manipulation}
Utilizing force information in planning and control for dexterous manipulation enables more robust manipulation by considering the quality of grasps in realizing object motions~\cite{li2016dexterous, zarrin2023hybrid}. Similarly, successful extrinsic manipulation, such as tilting and flipping objects, requires precise control of contact forces to prevent undesired sliding~\cite{chen2023synthesizing,cheng2022contact,wu2020r3t,lee2015hierarchical}. Most of the existing works focus on contact force and position planning and validate the method in simulation only~\cite{chen2023synthesizing, wu2020r3t, lee2015hierarchical}. \cite{oller2024tactile} performed hardware experiments on flipping and rotating different objects; however, their work relies on a complex manipulator setup, including both a force-torque sensor and a bubble tactile sensor, and interacts with the objects only using a single rigid object. Our work enables a dexterous hand equipped with tactile fingertips to directly interact with different objects, which opens up different capabilities of extrinsic manipulation as we can locate and control multiple contacts with different fingers.
\section{Formulation and Methodology}
\begin{figure}[htb]
  \centering
  \includegraphics[width=0.65\linewidth]{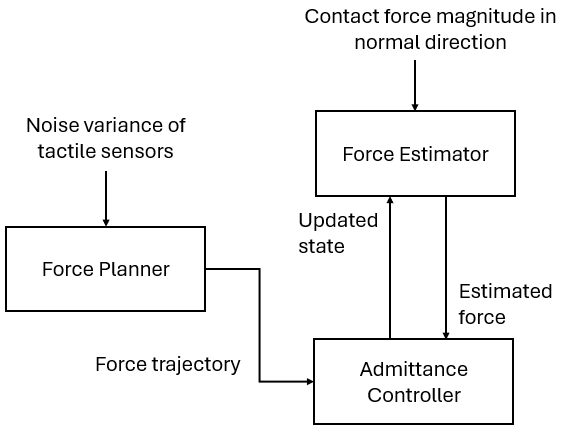}
  \caption{\textbf{System diagram of our proposed method}}
  \label{fig:overview}
\end{figure}
Our framework consists of three major components as shown in Fig.\ref{fig:overview}: a force planner that generates robust plans for the finger-object contacts with consideration of sensor error; a force estimator that uses tactile sensor reading, the robot state and the object pose to estimates all contact forces that would achieve force equilibrium under quasi-static assumptions; and a admittance controller that can track the planned kinematics and force trajectories robustly using estimated contact forces. In this section, we will first define the necessary concepts for our theoretical framework, and then use these concepts to address the problems of how to obtain force estimation and perform force planning with force observation uncertainty. We will end by describing the control architecture of our framework. The main notations used in this section are defined in Tab. \ref{tab:notations}

\begin{table}[tb!]
\caption{Nomenclature of the Proposed Framework}
\centering

 \begin{tabular}{l l} 
 \toprule
 Name & Notation \\ [0.5ex] 
 \midrule
 Number of extrinsic contacts & $n_\text{e}$ \\

 Number of intrinsic contacts & $n_\text{i}$ \\
 

 $i^{th}$ contact normal & $\bm{n}_i\in\mathbb{R}^{3}$ \\


 All contact forces & $\bm{f}\in\mathbb{R}^{(n_\text{i}+n_\text{e}) \times 3}$ \\

 Gravity wrench & $\bm{g}\in\mathbb{R}^6$ \\
 
 Extrinsic contact set & $\mathcal{C}_\text{ext}$ \\
 
 Intrinsic contact set & $\mathcal{C}_\text{int}$ \\

 Space of contact forces & $\mathcal{F}$ \\

 Force equilibrium constraints matrix & $\bm{A}_{fe}$ \\

 FE-basis & $\bm{B}_{fe}$ \\

 Coordinate in FE-basis & $\bm{x}_{fe}$ \\
 
 Constraint matrix & $\bm{C}$ \\

 Measurement cone edge weight & $\bm{m}$ \\

 Coordinate transformation matrix & $\bm{P}_{fe}$ \\
 \bottomrule
\end{tabular}
\label{tab:notations}
\end{table}

\subsection{Definitions}
We consider manipulation using a multi-finger dexterous hand; each fingertip is wrapped with a tactile sensor array that can provide contact location and contact force in the normal direction. We first define a contact $c$ on the object as a tuple of contact position, contact normal, and contact force. We denote the set of contacts between objects and tactile-enabled fingertips as intrinsic contacts $\mathcal{C}_\text{int}$. For the set of contacts between the object and the environment or other parts of the dexterous hand without tactile sensors, we denote it as extrinsic contact set $\mathcal{C}_\text{ext}$. We only consider quasi-static manipulation under gravity without other external forces, where the object is in force equilibrium, and all contact forces should cancel gravity all the time. We define key concepts for our paper as follows:
\begin{definition}[Space of contact forces]
Given a set of contact points including $n_i$ intrinsic and $n_e$ extrinsic contact points, we stack all forces into a vector $\bm{f}\in \mathbb{R}^{(n_i+n_e)\times 3}$, we define space of all possible forces as space of contact force $\mathcal{F}$.
\end{definition}
\begin{definition}[FE-plane]
All forces that can balance the object's gravity lie in a hyperplane in $\mathcal{F}$; we define this hyperplane as a Force-Equilibrium plane (FE-plane). Forces on FE-plane satisfy linear force equilibrium constraints:
\begin{equation}    
\bm{A}_{fe}^T\bm{f} = -\bm{g}
\label{eq:FE}
\end{equation}

Where $\bm{g}\in\mathbb{R}^6$ is gravity wrench and constraint matrix $\bm{A}_{fe}$ has the shape of $((n_i+n_e)\times3,6)$. As the number of effective constraints imposed on space of contact force is $\text{rank}(A_{fe})$, FE-plane has the dimension of $d_{fe} = (n_i+n_e)\times3-\text{rank}(A_{fe})$. 
\end{definition}
\begin{proposition}[FE-basis]
On the FE-plane, we can construct a set of orthonormal vectors that define a basis $\bm{B}_{fe}=\{\bm{b}_1,...,\bm{b}_{d_{fe}}\}, \bm{b}_i\in\mathbb{R}^{(n_i+n_e)\times3}$. Using this FE-basis, we can represent any force vector on the FE-plane using a weighted sum of these basis vectors $\bm{f}_{fe}=\sum_i^{d_{fe}}w_i\bm{b}_i+\bm{f}_0$, where $\bm{f}_0$ is a bias term, a particular solution on force equilibrium plane and the weight vector would be the FE-plane coordinate (FE-coord) $\bm{x}_{fe}$. Using the FE-basis we can construct a matrix combining subspace projection matrix and coordinate transformation matrix: $\bm{P}_{fe}=[\bm{b}_1;...;\bm{b}_{d_{fe}}]$ that project any vector in force space to FE-plane and express it in FE-coordinate: $\bm{x}_{fe} = \bm{P}_{fe}(\bm{f}-\bm{f}_0)$.
\end{proposition}
\begin{definition}[Constraint convex set]
The constraint convex set includes friction cone constraints and minimum force constraints for intrinsic contacts. Using a pyramidal friction cone, all constraints can be defined as linear inequality constraints in force coordinates; hence, the interior of the constraint set is always convex.
\begin{equation}
\bm{C}\bm{f}\geq\bm{d}
\end{equation}
Where constraint matrix is constructed from stacking linear friction cone constraints and linear minimum force constraints $\bm{C} = [\bm{C}_\text{frc}^T,\bm{C}_\text{minf}^T]$ and $\bm{d} = [\bm{0}_{[4\times(n_e+n_i)]},\bm{f}_\text{min}]$. We can also represent this constraint convex set in FE-coordinate when setting $\bm{C}_{fe} = \bm{P}_{fe}\bm{C}$:
\begin{equation}
\bm{C}_{fe}\bm{x}_{fe} \geq \bm{d}
\label{eq:FE_constraint}
\end{equation}
\end{definition}
\begin{proposition}[Perservation of force equilibrium]
For any force $\bm{f}\in\mathcal{F}$ on FE-plane, applying a delta force $\Delta\bm{f}$ on the plane will result in a new force $\Tilde{\bm{f}} = \bm{f}+\Delta\bm{f}$ still on FE-plane.
\end{proposition}
\begin{proof}
As both original force and delta forces are two vectors with the same basis, adding these two vectors will also result in a vector in the plane spanned by the same basis.
\end{proof}

For intrinsic contact, we can only measure the magnitude of non-negative contact force along the contact's normal direction. Therefore, assume there are only intrinsic contacts; all possible measurements lie inside a cone in force space. When extrinsic contacts are present, we can also assume there is a virtual sensor attached to the contact point that can measure force in the normal direction, but we cannot obtain readings from those sensors. Mathematically, we represent all possible measurements on all contact points as measurement cone (M-Cone)
\begin{definition}[M-Cone]
Let $\{\bm{n}_1,\bm{n}_2,..,\bm{n}_{n_i+n_e}\}$ be the set of contact normal vector. We construct M-Cone in force space as the weighted summation of contact normal vectors:
\begin{equation}
\mathcal{M} = \{\bm{f}:\exists \bm{m}, \bm{f}=\sum_j^{n_i+n_e} m_j\textbf{Extend}(\bm{n}_j)\}
\end{equation}
$\textbf{Extend}$ operator writes contact normal in full force vector dimension, elements corresponding to this particular contact normal are set to $\bm{n}_i$ while other elements are set to zero. For example, in the case there are 3 contact points $\textbf{Extend}(\bm{n}_1) = [n_1^x;n_1^y;n_1^z;0;0;0;0;0;0]$ and $\textbf{Extend}(\bm{n}_2) = [0;0;0;n_2^x;n_2^y;n_2^z;0;0;0]$. 
The measurement cone must go through the origin in force space when all weights $m_i=0,i=1,..,n_i+n_e$.

\begin{figure}[!h]
  \centering
  \includegraphics[width=0.8\linewidth]{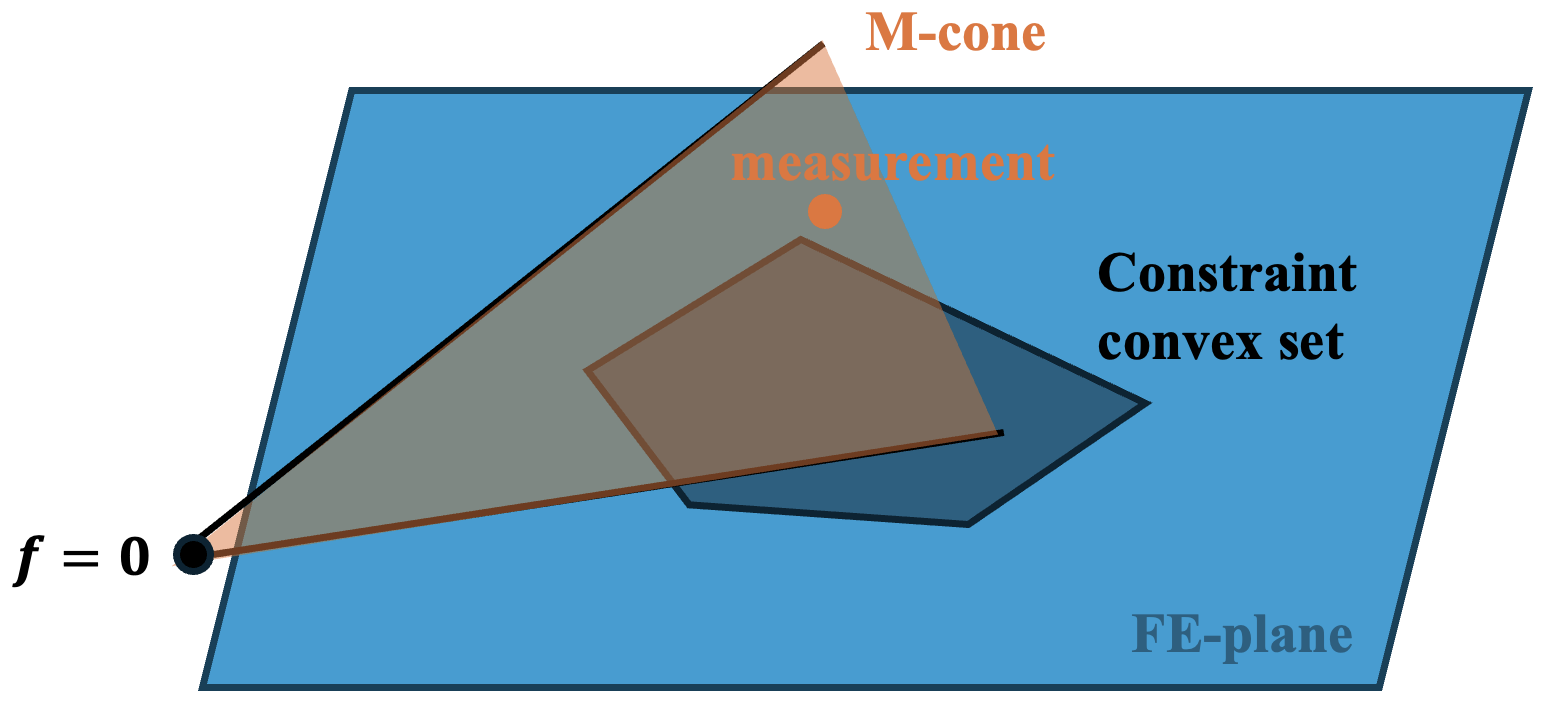}
  \caption{\textbf{FE-plane, M-Cone and Constraint convex set}}
  \label{fig:visualize}
\end{figure}
Fig.\ref{fig:visualize} provides a visualization on the FE-plane, M-Cone and constraint convex set in 3D as conceptual illustration.
\end{definition}

\subsection{Force Estimation}
\label{sec:Force_estimation}

Built upon these concepts, we can solve both force estimation and planning problems for different scenarios when there are different numbers of intrinsic and extrinsic contact points.
\begin{problem}[Force estimation for dexterous grasping]
We consider the problem of finding the best force estimate in dexterous grasping. This problem involves only intrinsic contacts and to achieve any target wrench, at least two intrinsic contact points are needed ($n_i\ge 2$)
\end{problem}
To solve this problem, we need to find a force $\bm{f}_\text{est}$ on the FE-plane that has a minimum Euclidean distance to the observed normal forces $\bm{f}_\text{m}$ on M-Cone. We can achieve this by projecting $\bm{f}_\text{m}$ on to FE-plane:
\begin{equation}
\begin{split}
\bm{x}_\text{est} = \bm{P}_{fe}(\bm{f}_m - \bm{f}_0) \\
\bm{f}_\text{est} = \bm{P}_{fe}^T\bm{x}_\text{est}+\bm{f}_0
\end{split}
\end{equation}
Although under ideal conditions, an observation is obtained by projecting from the  FE-plane to the measurement cone. However, in reality, force estimation under this assumption requires back-tracing projection to FE-plane, which will magnify normal force measurement error. Our projection allows changes to normal force magnitude and practically gives similar results as we will show in the experimental section.
\begin{problem}[Force estimation for extrinsic manipulation]
In extrinsic manipulation, there exist non-zero extrinsic contact points ($n_e>0$), which don't have normal force readings. We need to estimate intrinsic forces based on intrinsic and extrinsic contact locations as well as normal force readings from intrinsic contacts.
\end{problem}
\begin{figure}[!h]
  \centering
  \includegraphics[width=0.6\linewidth]{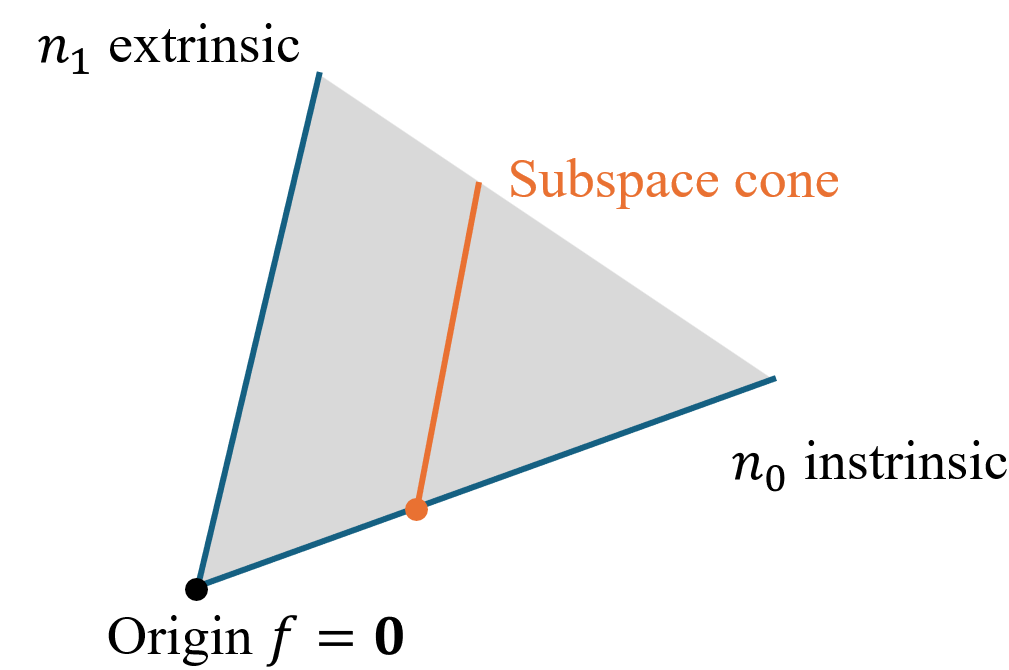}
  \caption{\textbf{Illustration of measurement sub-space cone} Assume two contact points are involved, one intrinsic, one extrinsic. Orange ray shows measurement sub-cone as we can only measure force magnitude at the intrinsic contact location.}
  \label{fig:ext_subcone}
\end{figure}
Only from intrinsic force measurement we can no longer obtain a point of measurement on M-Cone; instead, all possible measurement spans a sub-space cone on measurement cone as shown in Fig.\ref{fig:ext_subcone}. To find the best force estimate, we need to find a force $\bm{f}_\text{est}$ on the FE-plane that is close to this sub-space cone. Meanwhile, as extrinsic contact points only provide minimum contact forces to counteract gravity and intrinsic contact forces, such as intrinsic contact forces and gravity, we also need $\bm{f}_\text{est}$ to have minimum magnitude, i.e. close to the origin. Therefore, to estimate contact forces, we need to solve a quadratic program (QP) that minimizes the sum of the distance to the sub-space cone and the origin. To lower the complexity of solving QP, we represent $\bm{f}_\text{est}$ using FE-coord $\bm{x}_\text{est}$, thus distance to the origin can be expressed as:
\begin{equation}
d_\text{origin} = ||\bm{P}_{fe}^T\bm{x}_\text{est}||_2\\
\end{equation}
To obtain the distance to the sub-space cone, we first need to construct basis $\bm{B}_\text{sub}$ for this measurement sub-space cone. As each extrinsic contact point contributes to one independent DoF in the sub-space cone, we first compute $n_e$ linearly independent force vectors with each one containing non-zero terms from one of the extrinsic contact normals and zeros in other directions $\bm{f}^i_\text{sub}$. $\bm{B}_\text{sub}$ can be constructed from $\{\bm{f}^i_\text{sub}\}$ using QR decomposition. After obtaining force vectors, we can find a weight $\bm{w}$ of linear combination $w_i\bm{f}^i_\text{sub}+...+w_i\bm{f}^{n_e}_\text{sub}$, and force estimate can be expressed as a projection of this linear combination onto the FE-plane $\bm{f}_\text{est} = \bm{P}_{fe}\sum_i^{n_e} w_i\bm{f}^i_\text{sub}$. To get the best estimate, we need to solve $\bm{w}$ that minimizes the distance to the origin and subspace plane spanned by $\{\bm{f}^i_\text{sub}\}$:
\begin{equation}
\begin{split}
\min_{\bm{w}} ||\bm{f}_\text{est}|| + ||\bm{B}_\text{sub}\bm{f}_\text{est}|| \\
s.t. \bm{f}_\text{est} = \bm{P}_{fe}\sum_i^{n_e} w_i\bm{f}^i_\text{sub}
\end{split}
\end{equation}

\subsection{Force Planning}
\label{sec:Force_planning}
Using above mentioned method, we can obtain a force measurement that is guaranteed to achieve force closure but may have errors compared to ground truth force. If we have prior knowledge of the noise range of each tactile sensor, we can bound the error using an ellipsoid on FE-plane. Therefore, when planning forces for dexterous grasping or extrinsic manipulation, we need to consider the potential error and make a robust plan that remains feasible regardless.

\begin{definition}[Trusted measurement ellipsoid]
For each contact point, assume normal force measurement error is independent and obeys Gaussian distribution $\epsilon_i\sim\mathcal{N}(\sigma_i, 0)$. If we consider measurement within 1$\sigma$ from ground truth as a trusted region, all trusted measurements can be parameterized as an ellipsoid on M-cone
\begin{equation}
\begin{split}
\bm{m}^T\bm{D}\bm{m}\le 1\\
\bm{D} = \text{Diag}(1/\sigma^2_0,...,1/\sigma^2_{n_i+n_e})
\end{split}
\end{equation}
When there are extrinsic contacts, we set the variance of extrinsic contacts to infinity. While the quadratic formula still holds, the shape becomes an ellipsoid as one axis has been infinitely elongated.
\end{definition}

\begin{proposition}{(Propagating ellipsoid to FE-plane)}
Given an ellipsoid with center $\bm{c}$ expressed in M-cone basis, we can project it onto the force equilibrium plane using a projection matrix $\bm{E}$
\begin{equation}
\begin{split}
(\bm{x}_{fe}-\bm{c})^T\bm{M}(\bm{x}_{fe}-\bm{c})\\
\bm{M}=(\bm{E}^\dagger)^T\bm{D}\bm{E}^\dagger\\
\end{split}
\end{equation}
To construct measurement projection matrix $M$ we consider the following: Any force vector on measurement cone can be expressed as $\bm{f} = \sum^{n_i+n_e}_im_i\textbf{Extend}(\bm{n}_i) = \bm{N}_{M-cone}\bm{m}$ where $\bm{N}$ is constructed from extended normal vectors. We project this force to FE-plane:
\begin{equation}
\begin{split}
\bm{x}_{fe} = \bm{P}_{fe}\bm{N}_{M-cone}\bm{m}\\
\bm{x}_{fe} = \hat{\bm{E}}\bm{m}\\
\end{split}
\end{equation}
As the system is under-determined, we need to set $\hat{\bm{E}} = \bm{Q}\bm{E}; Q,R = QR(E^T)$ (QR decomposition) to get minimum norm solution when computing the projection.
\end{proposition}

\begin{problem}{(Force planning for dexterous grasping)}
Given a set of intrinsic contact points, planning a safe dexterous grasp requires solving center $c$ of a trusted measurement ellipsoid on the force-closure plane such that the entire ellipsoid satisfies friction cone and minimum force constraints.
\begin{equation}
\begin{aligned}
&\min_{\bm{c}}\ 0 \\
&s.t.\forall \bm{x}_{fe}, (\bm{x}_{fe}-\bm{c})^T\bm{M}(\bm{x}_{fe}-\bm{c})\le1\\
&\bm{C_{fe}}\bm{x}_{fe} \ge \bm{d}
\end{aligned}
\label{eq:grasp}
\end{equation}
\end{problem}
Directly considering the analytical ellipsoid is challenging, thus we rewrite constraints of $\bm{x}_{fe}$ as constraints on $\bm{c}$, for each linear constraint of from constraint matrix $\bm{C}_{fe}$
\begin{equation}
\bm{C}_{fe}[i]\bm{x}_{fe}\ge d_i
\end{equation}
It can be rewritten as
\begin{equation}
-\bm{C}_{fe}[i]^T\bm{c}+\sqrt{-\bm{C}_{fe}[i]^T\bm{M}^{-1}-\bm{C}_{fe}[i]}\le -b_i
\end{equation}
We can then write constraints on the ellipsoid center as:
\begin{equation}
\begin{aligned}
\Tilde{\bm{C}}_{fe}[i] &= - \bm{C}_{fe}[i]\\
\Tilde{\bm{b}}[i] &= -\bm{b} - \sqrt{-\bm{C}_{fe}[i]^T\bm{M}^{-1}-\bm{C}_{fe}[i]}
\end{aligned}
\end{equation}
Therefore, we can solve the following linear feasibility problem of $\bm{c}$
\begin{equation}
\begin{aligned}
&\min_{\bm{c}}\ 0 \\
&s.t. \Tilde{\bm{C}}_{fe}\bm{c}\le\Tilde{b}\\
\end{aligned}
\label{eq:grasp_reform}
\end{equation}
The force for each fingertip can be obtained by projecting optimal $\bm{x}^*_{fe}$ back to the space of contact force.

\begin{problem}{(Force planning for extrinsic manipulation)}
Given a set of intrinsic contact points and a set of extrinsic contact points, planning a set of safe intrinsic forces can be formulated similar to \ref{eq:grasp} with a cost term added to minimize total force magnitude
\begin{equation}
\begin{aligned}
&\min_{\bm{c}}\ \bm{P}^T_{fe}\bm{c} + \bm{x}_0 \\
&s.t.\forall \bm{x}_{fe}, (\bm{x}_{fe}-\bm{c})^T\bm{M}(\bm{x}_{fe}-\bm{c})\le1\\
&\bm{C_{fe}}\bm{x}_{fe} \ge \bm{b}\\
\end{aligned}
\end{equation}
\end{problem}

To solve the problem, we can apply a similar re-formulation as in \ref{eq:grasp_reform}, because we don't have all contact force readings from the tactile sensor, matrix $\bm{M}$ becomes low rank and pseudo inversion should be applied in the place of the matrix inverse.

\subsection{Controller}
\label{sec:controller}
The control loop architecture is presented in Fig.~\ref{fig:controller_basic}. Starting with the plant, we consider the system comprised of the Hand-Object-Environment, where we can control the robotic hand desired joint positions, which will affect the interaction forces between the fingertips and the object, and between the object and the environment. The state of the system will be given by the hand's joint angles $\bm{q}$, the contact forces $\bm{f}$ and the object's pose. We use the fingertips' tactile sensor to measure the intrinsic contact forces. Due to the sensor's measurement error we use the proposed force estimator in \ref{sec:Force_estimation} to get an improved force observation that we use to close the control loop. Considering the initial multi-finger grasp,  we can compute the initial contacts' locations and normals, and combine them with the object's properties to compute the desired force trajectory that will feed the controller desired force $\bm{f}_{des}$ as described in \ref{sec:Force_planning}.
\begin{figure}
    \centering
    \includegraphics[width=0.9\linewidth]{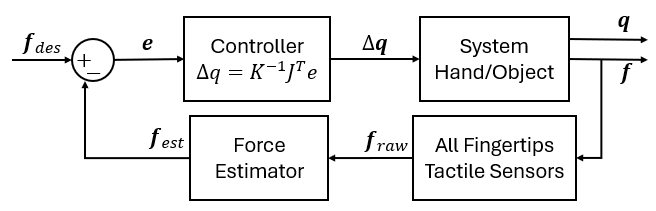}
    \caption{Block diagram of the control system.}
    \label{fig:controller_basic}
\end{figure}

We use the error $\bm{e}$ between the desired forces and the observations at each contact point along with the fingertip's Jacobian $J$ to compute direction of motion of the fingertips as
\begin{equation}
    \Delta \bm{q} = K^{-1}J^T(\bm{q})\bm{e}
\end{equation}
Where $K$ is a positive-definite gain matrix. We can use this $\Delta\bm{q}$ to update the desired joint angles of the hand in the plant, which is realized through the hand's low-level PD controller.

\section{Experiments and Results}
We run both simulation and hardware experiments to evaluate the performance of our proposed method. For the algorithm implementation, we approximate the friction cones through a 12 sided pyramidal geometric approximation. Also, the constraint matrix $A_{fe}$ is computed through single value decomposition to construct the basis. For the experiments, we use the Allegro Hand V$4$ which has $4$ fingers, each with $4$ Degrees of Freedom (DoF), and we replaced the original fingertips of the hand with custom tactile sensor fingertips by Touchlab Limited Inc \cite{Touchlab}. We use MuJoCo to simulate the arm, hand, and objects' kinematics, dynamics, and contact interactions. For the hardware experiments, we use a Franka FR3 arm to position and move the Allegro hand. An Optitrack camera system is used to measure the $6D$ pose of the robotic hand, arm, and objects, which has an accuracy of $0.4mm$ and $0.05^\circ$. The interaction between the fingertips and the objects is measured using the tactile fingertips which output normal forces at the contact location. The hardware setup and experiment objects are shown in Fig.~\ref{fig:Hardware_setup}. 

In each experiment, we start by closing the multi-finger hand to grasp the object and acquire the initial tactile readings and use these initial measurement of forces to determine the contact location and normal. Then, using the object's mass, Center of Mass (CoM), and friction, and assuming a quasi-static motion, we compute the desired force, $\bm{f}_{des}$, for each contact using the Newton-Euler equations of 3D rigid-body dynamics to realize the object motion (hold the object still in case of grasping) as shown in eq.~\ref{eq:FE} and considering the minimum force constraints as in eq.~\ref{eq:FE_constraint}. While we can measure the object's mass, for CoM and friction coefficient we use an offline parameter estimation. If no good estimation of the friction coefficient or CoM is available, a more conservative estimation value that shrinks the friction-cone can be used so the final computed grasp will be inside the real friction-cone. Also, it should be noted that the object's parameters are mainly used during force planning, and only the mass is required during the force estimation algorithm. To compare the performance improvement due to the force estimator, we switch the feedback signal between the fingertips' raw measurements, $\bm{f}_{raw}$ and the estimated force values computed by our proposed force estimator, $\bm{f}_{est}$. We control the desired contact force in an admittance controller framework as described in \ref{sec:controller}. 

\begin{figure}
    \centering
    \includegraphics[width=0.75\linewidth]{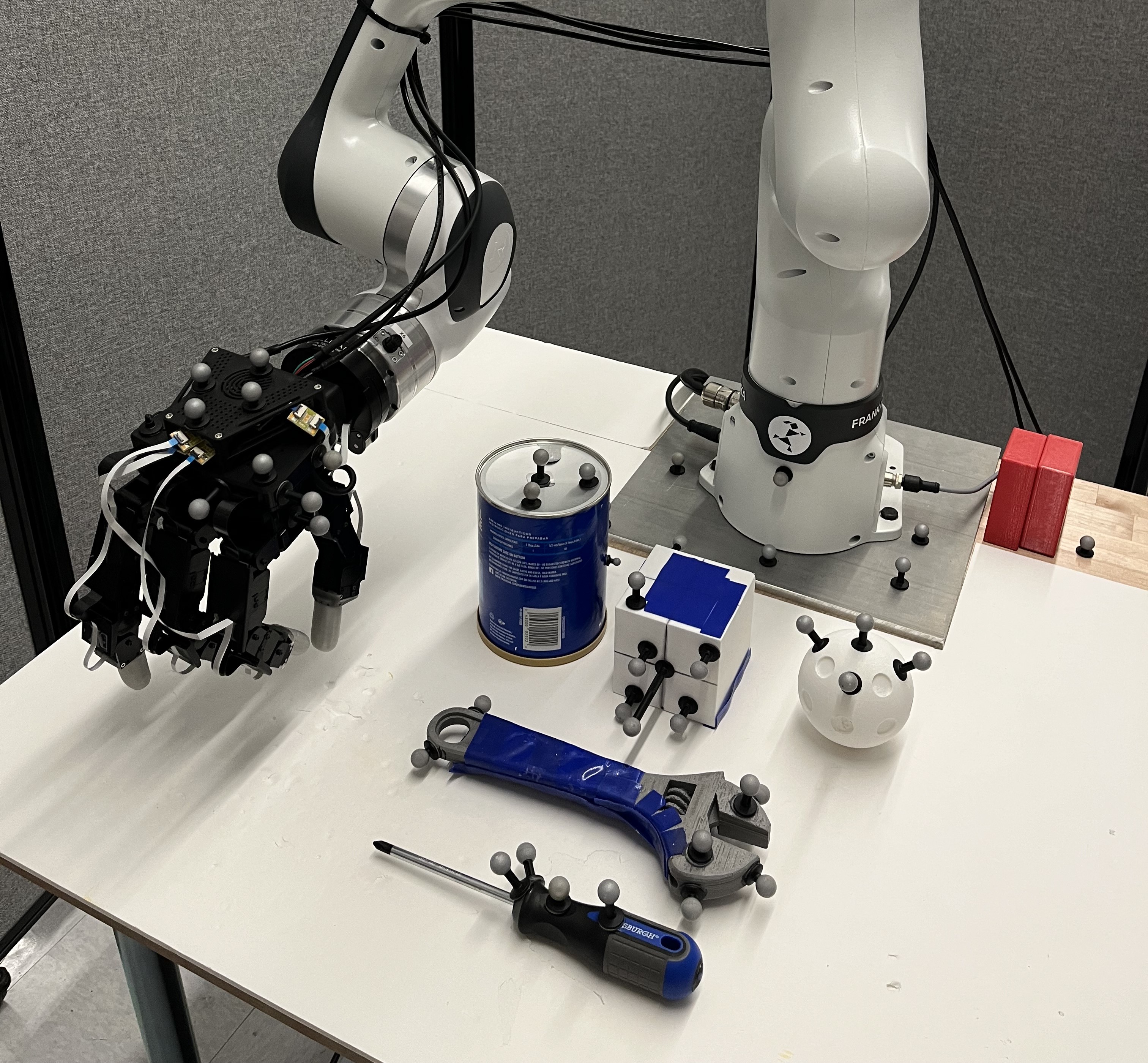}
    \caption{Hardware setup including Allegro hand equipped with Touchlab fingertips, and Franka arm. The $3D$-printed sphere and wrench, as well as the cylindrical can are used for dexterous grasping experiments, while the $3D$-printed cube and a real screwdriver are used for extrinsic manipulation experiments.}
    \label{fig:Hardware_setup}
\end{figure}

\subsection{Force Measurement via Tactile Array}
Each tactile fingertip part has an array of $42$ small 1D piezo-electric sensors (taxels) that are placed on the tip, body, and sides of the fingertip. The position of each taxel and its orientation (normal vector) relative to the fingertip base are known. We define the \textit{raw} measurement as the values provided by each taxel which represent the calibrated normal force reading of each taxel. Assuming point contact, the distributed forces on each fingertip can be represented by a single raw force, $\bm{f}_{raw}$, located at the mean location of distributed forces, i.e. weighted average of taxels positions times their measured force. Similarly, the contact normal of each fingertip is the normalized weighted average of the normal vectors times their measured force. Each taxel has been calibrated to measure a force between $0.1-20N$ under full contact conditions.

Touchlab's tactile fingertips are responsive even to small forces and their dense coverage of the fingertip surface, including sides and the top, allows a reliable tactile-based dexterous manipulation. However, despite their higher performance compared to many of the state-of-the-art commercialized tactile sensors we have tried, they are still in an early stage in terms of providing accurate force measurement. While each taxel is calibrated, due to approximations in calibration as well as the discrepancy between calibration conditions and real-world contact-rich manipulation scenarios, such as partial contact in real-world vs full-contact during calibration or contact forces not aligned with the taxel's normal (applying force at an angle), measurements have considerable errors. These discrepancies lead to measurement errors in the range of $0.1 \sim 1 N$ per taxel, which can accumulate when computing the fingertip's raw force measurement. The characteristics of individual taxels can be summarized by the following:
\begin{itemize}
    \item Activation threshold: Minimum required force that the taxel will measure ($0.1 - 0.5 N$). Effect: As additional taxels come into contact, there is a discrete jump in the measurements when a force increases over a taxel's activation threshold.
    \item Hysteresis: Force offset when releasing contact ($0.0 - 0.2 N$). Effect: As contact shifts, taxels that lose contact and should not sense any force, might output a small value.
    \item Force error: Due to partial contact or when the interaction force is not aligned with the taxel ($0.1 - 0.5 N$). Effect: When adding the taxels' forces, we have an offset between the ground truth force and the measured value.
    \item Taxel's noise: inherent noise of the piezo-electric taxel at steady state ($0.01 - 0.03 N$). Effect: No major effects once contact has been established.
\end{itemize}

We use these characteristics in simulating our tactile responses in MuJoCo to have an accurate representation of the error and noise of our tactile fingertips when testing our algorithm in simulation. 
A comparison between the ground truth of the applied force (provided by ATI Nano17 Force/Torque sensor), taxel's real-world measurement on hardware, and the simulated values in MuJoCo are shown in Fig.~\ref{fig:tactile_comparison}. As seen in this figure, the simulated tactile data is able to closely represent the characterization of the real tactile data. The main discrepancy between the simulated and hardware measurements is that we consider one uniform characterization for all the fingertip taxels in simulation independent of contact size or contact location, while in real taxels the measurement error could change based on the contact location and contact area, and the error is different between different fingertips. Nevertheless, the simulated behavior's resemblance to real-world readings allows us to test our algorithm in simulation.

\begin{figure}
    \centering
    \includegraphics[width=0.9\linewidth]{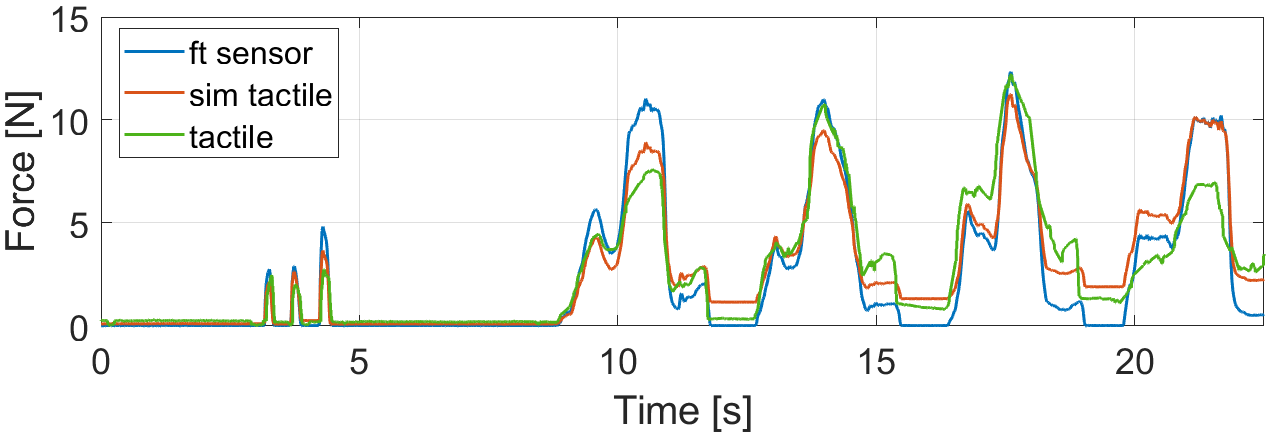}
    \caption{Contact force comparison between ground truth (F/T sensor), hardware tactile measurement and simulated tactile measurement}
    \label{fig:tactile_comparison}
\end{figure}

 

\subsection{Simulation Results}
Fig.~\ref{fig:simulation_geodex} represents examples from grasping (wrench) and extrinsic manipulation (cube) in the simulated environment in MuJoCo. Associated hardware experiments will be presented in \ref{hardware}. For the simulation, we either measured or approximated the respective objects' properties to their real-world counterpart. The simulation uses the same values as the hardware for the hand joints' PD gains. For the wrench grasping simulation, we consider a $0.3kg$ object, with a friction of $0.9$ and a tactile sensing noise variance of $0.5N$. We compared the controller when using the estimated force values against the raw measurements, with the results shown in  Fig.~\ref{fig:sim_wrench_comparison}. When using the estimated force values, the forces converge to their respective desired values. For the experiment using the raw measurements, we can see that forces are not able to converge as the thumb is under the desired force but the index and middle fingers are over the desired force. This exemplifies the problem of the error in measurements, where we need the thumb to apply more force to reach to its desired value but at the same time the index and middle finger should decrease their force. Since the thumb opposes the forces applied by the index and middle finger, thus they have to increase or decrease together, thus the equilibrium cannot be achieved.

\begin{figure}
    \centering
    \includegraphics[width=0.85\linewidth]{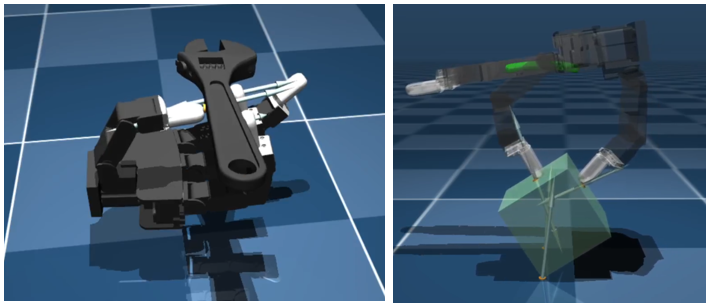}
    \caption{Simulation environment in MuJoCo for the wrench grasp and cube extrinsic manipulation. Simulated interaction forces at the contact points are displayed.}
    \label{fig:simulation_geodex}
\end{figure}
\begin{figure}
    \centering
    \includegraphics[width=0.99\linewidth]{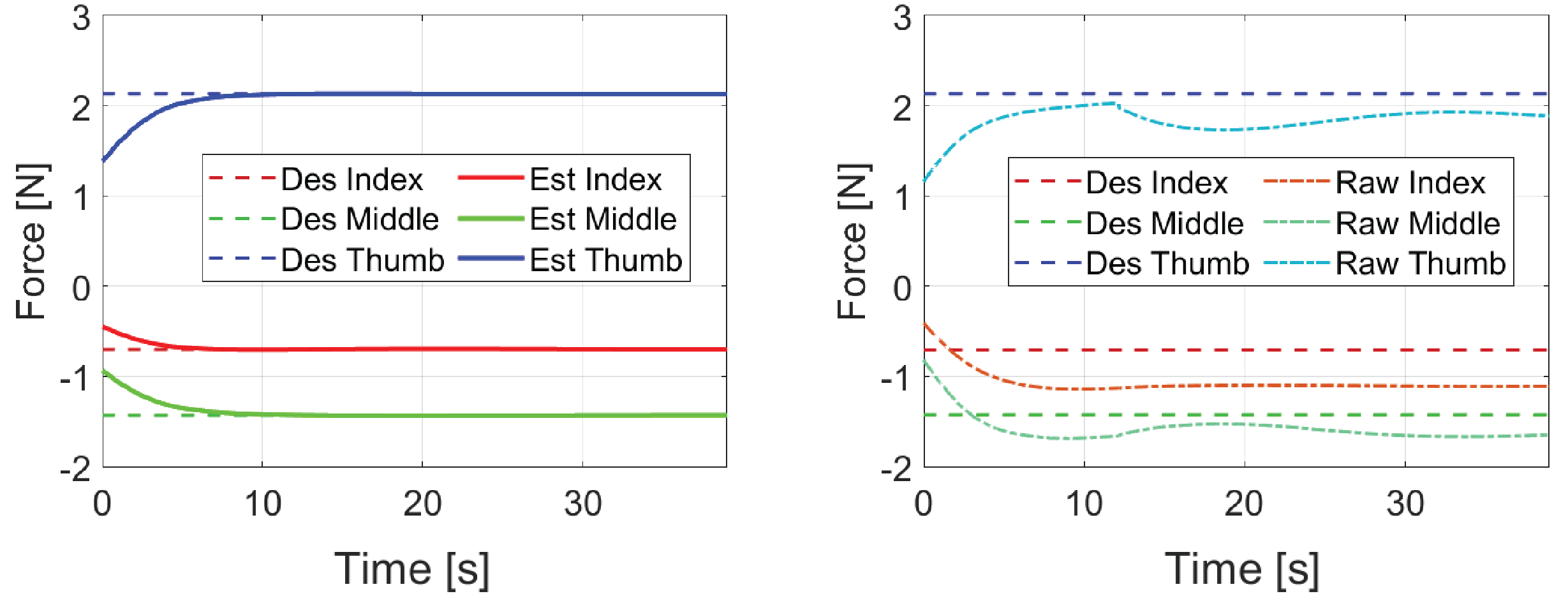}
    \caption{Grasping wrench comparison using the force estimation (left) and the raw measurements (right).}
    \label{fig:sim_wrench_comparison}
\end{figure}
For extrinsic manipulation simulation, we use an $8cm$ cube, with a mass of $0.15kg$, and friction coefficients of $0.8$ at the fingertip contacts and $0.6$ at the table contact. Similar to the wrench simulation, we consider the same hand joints' PD gains and tactile sensing noise of $0.5N$. Here we can show that our force planner can generate a sequence of fingertip contact location and forces and then have our controller follow the trajectory successfully as seen in Fig.~\ref{fig:cube_sim}. It takes a couple of seconds to accomplish the force control using the force estimation, leading to an increase in the RMS error at the start of the sequence. Once the controller converges to the desired forces, the system starts tracking the desired cube's yaw angle with an RMS error under $1^\circ$.
\begin{figure}
    \centering
    \includegraphics[width=0.6\linewidth]{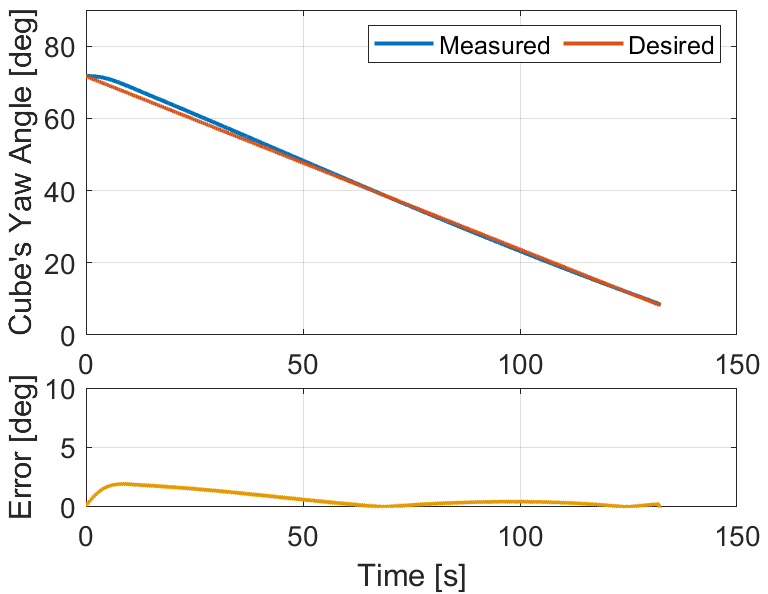}
    \caption{Cube turning in simulation}
    \label{fig:cube_sim}
\end{figure}

As part of determining the benefits of our geometric framework for force planning and estimation, we compared the computation time of our method, in simulation, to directly running optimization using Second Order Cone Programming (SOCP) at different lengths on an example grasp force planning of a 3-finger grasp experiment of an object. The results are presented in Tab.~\ref{tab:simulation_speed}. According to the results, we can see an improvement in computational speed with about a $14\times$ speed-up using our framework over the SOCP.

\begin{table}[h]
    \centering
    \begin{tabular}{c|c|c}
         Method & 100 steps & 300 steps  \\
         \hline
         SOCP & $4.81s$& $13.36s$ \\
         Geometric & $0.31s$ & $0.97$
    \end{tabular}
    \caption{Execution time comparison between SOCP and Geometric optimization.}
    \label{tab:simulation_speed}
\end{table}

\subsection{Hardware Results} \label{hardware}
\subsubsection*{\textbf{Dexterous Grasping}}
We will first introduce the underlying problem of inaccurate force measurements in hardware by presenting a 2-finger pinch grasp of a sphere. Then we will study the performance of our algorithm using two experiments, 1) a 3-finger grasp of a wrench, and 2) a 4-finger grasp of a large cylinder. With a three and four-finger grasp, we can ensure that force equilibrium can be accomplished and we can compare the controller's performance. The used grasps are shown in Fig.~\ref{fig:All_Dexterous_grasps}.

\begin{figure}
    \centering
    \includegraphics[width=0.99\linewidth]{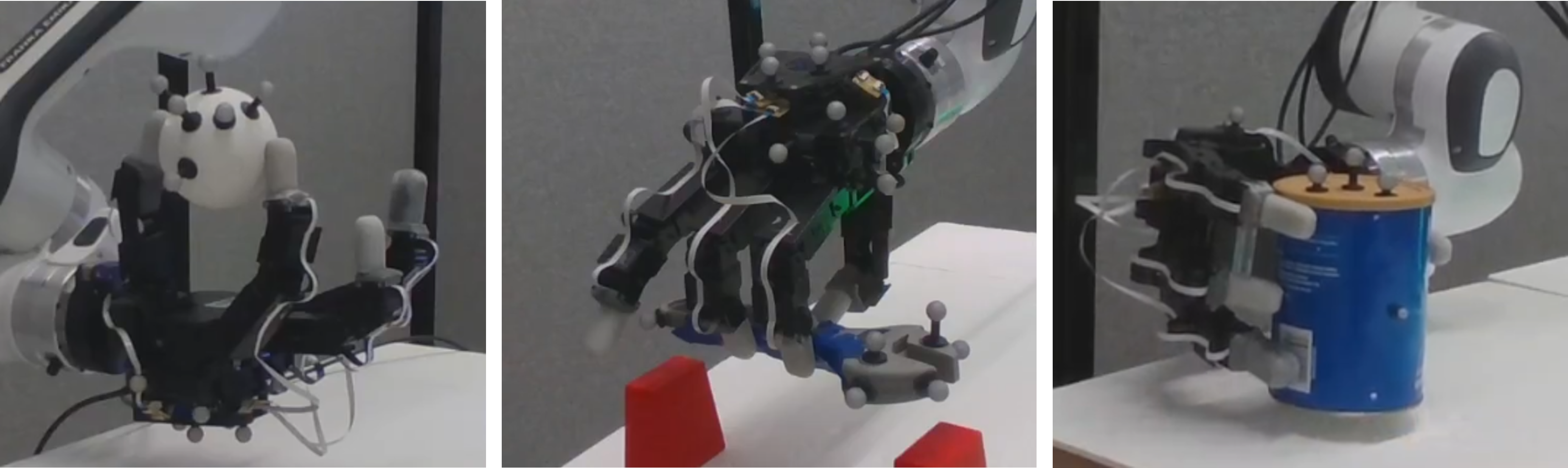}
    \caption{Successfully grasping different objects using the proposed method}
    \label{fig:All_Dexterous_grasps}
\end{figure}


\subsubsection{2-finger grasp}
To represent the underlying idea of this work in a simplified scenario, we consider a two-finger grasp of a sphere with a mass of $0.082kg$. 
To accomplish force balance for this example, we need the contact forces to have equal magnitude and opposite direction. Given the object's mass, CoM, contact locations, and friction coefficients at the contact points, we use our force planner to compute the desired force that each fingertip should apply to hold the object (as in eq.~\ref{eq:FE}) without over-pressuring it (following constraint in eq.~\ref{eq:FE_constraint})  
\begin{figure}
    \centering
    \includegraphics[width=0.99\linewidth]{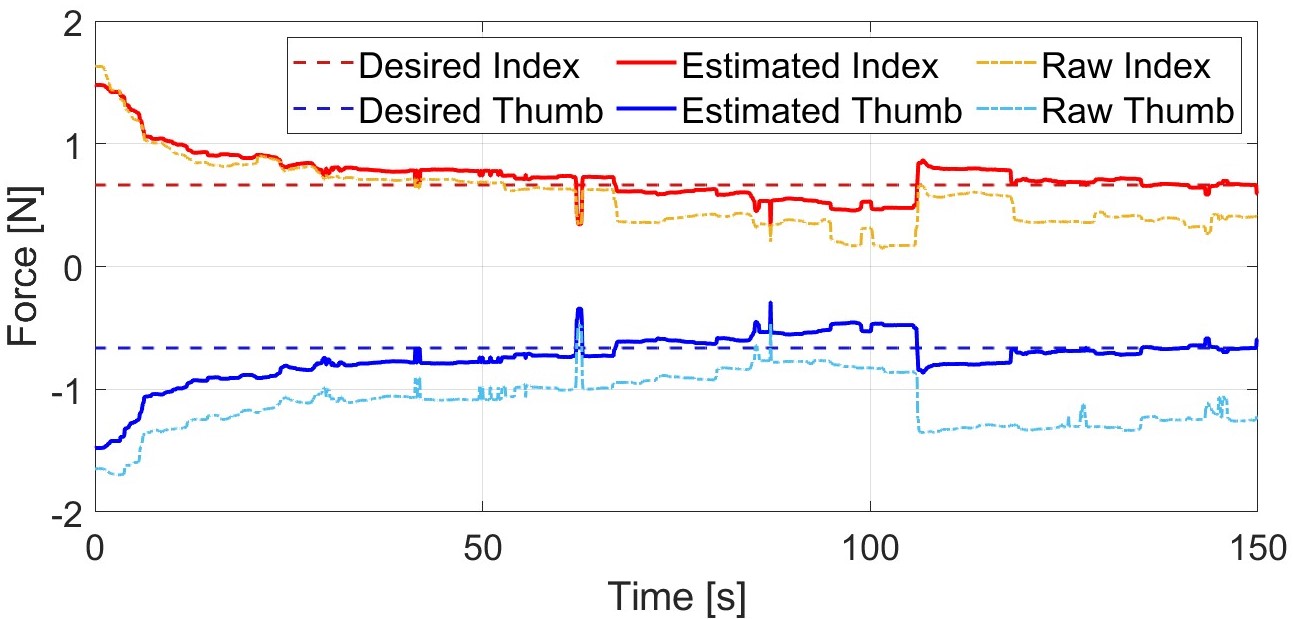}
    \caption{2-Finger grasp of sphere.}
    \label{fig:sphere_grasp}
\end{figure}
Then, during the control phase, using the force estimations as a feedback signal, the controller can control the forces to their respective desired values. As shown in  Fig.~\ref{fig:sphere_grasp}, the estimated contact forces converge to the desired value with a mean error of $0.1N$. In this case, we have achieved force balance as the estimated forces have converged to their respective values and there is no apparent motion on the object. Now, we can take a closer look at the raw measurements, which throughout the experiment show a close correlation but do not have equal magnitude. If we were to control the system using the raw measurements as the feedback signal, it becomes apparent that we would have to break the force balance that was accomplished. Analyzing the end of the experiment, we have raw measurements for the thumb $f_{thumb} \approx -1$ and index finger $f_{index} \approx 0.6$, which would lead to a control output where we want to decrease the applied force by the thumb but increase the applied force by the index. Since we only have two fingers, either both have to increase the applied force or both have to decrease it to keep the force balance. Thus, using the controller based on the raw measurements would break the force balance eventually and the object would start to move.

\subsubsection{3-finger grasp}
We use a 3D printed wrench of $300 gr$ for this experiment. Similar to the 2-finger grasp, given the object properties and contact locations, the force planner computes the desired forces required for each fingertip to hold the object in force equilibrium. 
Fig.~\ref{fig:wrench_grasp} is a representative result of this experiment. We start by controlling the system based on the estimated forces. Within the first 5 seconds, the system converges to the desired values, with a mean error between the fingertips' desired force and the estimated values of $0.09\pm 0.04 N $. We hold the grasp for $20s$ to show that force equilibrium is achieved and the object pose remains static. Then, the object is lifted successfully at the $25s$ without any change in the desired object orientation and it is placed back successfully at $35s$. At $40s$, we switch the controller to use the raw measurements as observation. We can see that while the system tries to control the contact forces of all fingers to reach to the desired values, it is not successful. It can be seen that equilibrium cannot be achieved since while the thumb and middle finger have achieved forces close to desired values within an error of $0.2\pm 0.1N$, any attempt to reduce the error of the index finger by increasing its generated force would yield a reactive force on the thumb that will push thumb away from its desired value. Thus an error between $0.5 - 0.9 N$ is observed for the index finger. This force imbalance while trying to control the raw measurements, introduces unintended object movements such as tilting and thus losing grasp while attempting to lift the object.

\begin{figure}
    \centering
    \includegraphics[width=0.99\linewidth]{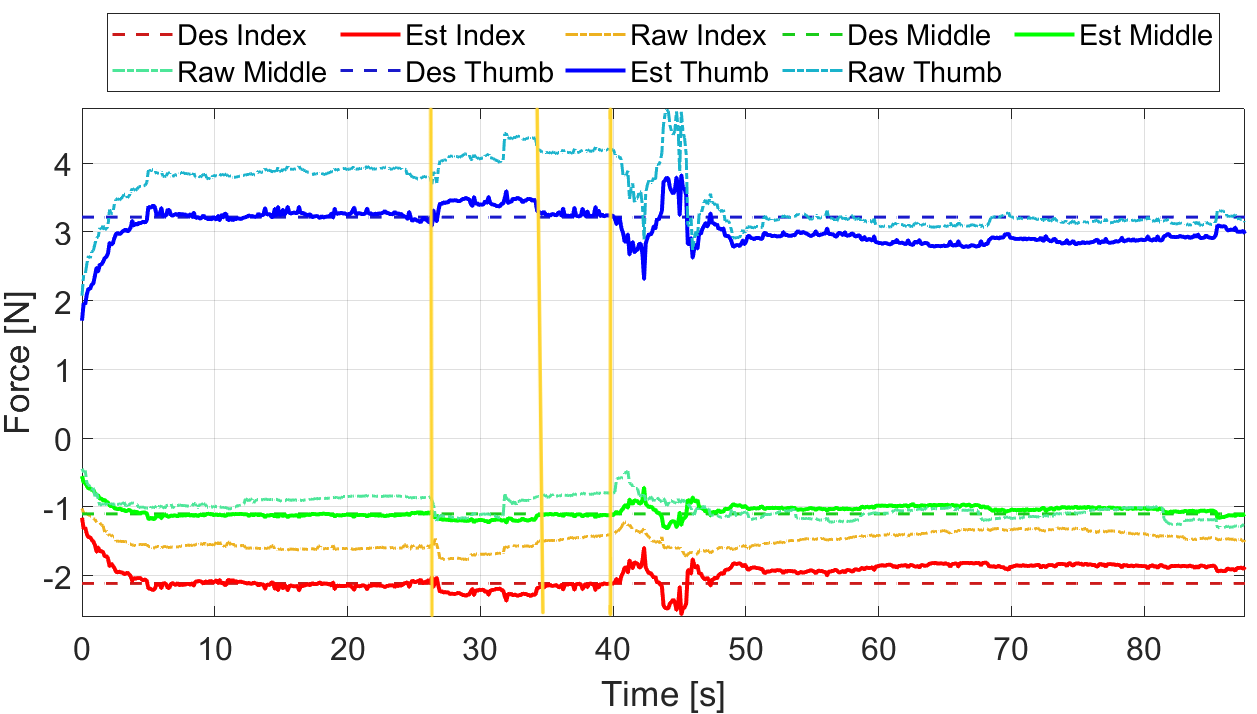}
    \caption{3-Finger grasp of wrench. From $t=0s$ to $t=27s$ we control the estimated forces to the desired values for each finger. The controller converges and keeps the forces around the desired value. From $t =27s$ we lift the object to show a successful grasp, at $t=33s$ we put the object back down and keep stable control. From $t=40s$ we try to control the raw force values. The system is not able to control all of the fingers to the desired force and the object is eventually dropped.}
    \label{fig:wrench_grasp}
\end{figure}

We performed the grasp and lift experiment of the wrench $20$ times, $10$ using the estimated force, and $10$ using the raw measurements. We will con sider a grasp to be successful if the object does not undergo a noticeable movement (more than $1^\circ$). The success rate along with the mean and standard deviation of the force error at the contact points for the success and failure cases is presented in table~\ref{tab:experiments}. 
We can see that using the estimated forces increases the success rate considerably compared to the raw measurements experiments. 
For the two experiments that failed using estimated forces, tactile sensors were not reacting to contact despite an increase in the exerted forces thus leading to overexertion of forces. In the two success cases using the raw measurements, we can see that the mean force error was very small. This can be the case when by chance the individual fingertips read similar measurements and close to the estimated values, which allow the fingers to accomplish a successful grasp.


\begin{table}[]
    \centering
    \begin{tabular}{c|c|c|r c}
        Object & Control & Success rate& \multicolumn{2}{c}{Mean \& std force error} \\
        \hline
        \multirow{4}{*}{wrench}& \multirow{2}{*}{$f_{est}$} & \multirow{2}{*}{$80 \%$}& success: &$0.09 \pm  0.04N$ \\
         &  & & failure: &$0.73 \pm  0.26N$ \\
         \cline{2-5}
         & \multirow{2}{*}{$f_{raw}$} & \multirow{2}{*}{$20 \%$}&success: &$0.19 \pm  0.16N$\\
         &  &&failure: &$0.83 \pm  0.57N$\\
         \hline
        \multirow{4}{*}{cylinder}& \multirow{2}{*}{$f_{est}$} &\multirow{2}{*}{$60 \%$}&success: &$0.35 \pm  0.17N$ \\
         & &&failure: &$0.57 \pm  0.31N$\\
         \cline{2-5}
         & \multirow{2}{*}{$f_{raw}$}&\multirow{2}{*}{$0\%$}& success: &N/A\\
         &   & &failure: &$0.82 \pm  0.45N$
    \end{tabular}
    \caption{Success rate for wrench and cylinder grasp experiments with the mean and std of the force error of the grasps when it was successful and when it failed.}
    \label{tab:experiments}
\end{table}


\subsubsection{4-finger grasp}
For this experiment, we use a cylinder with a mass of $0.360kg$ and a diameter of $0.11m$. The contact locations on the object are chosen in such a way that the controller has to account for the force balance in all directions. Similarly, we show that the system can control the estimated forces to their desired values as shown in Fig~\ref{fig:Cylinder_grasp}. 

\begin{figure}
    \centering
    \includegraphics[width=0.99\linewidth]{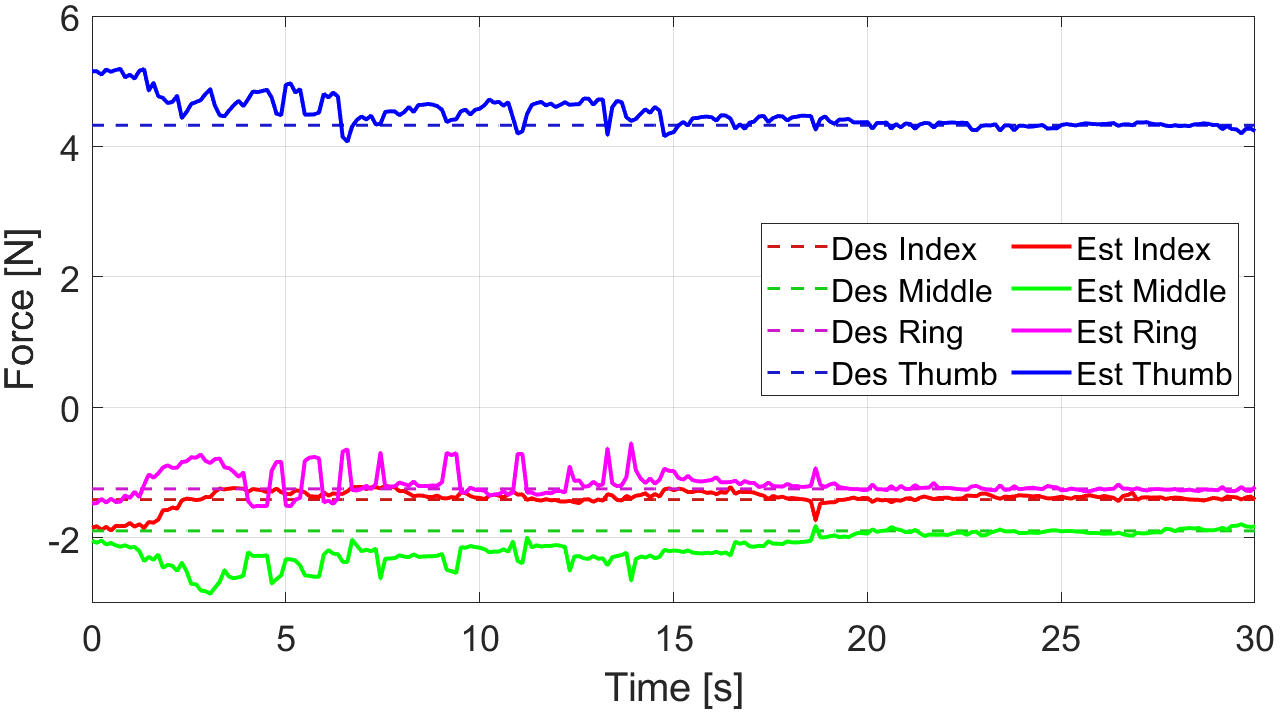}
    \caption{4-Finger grasp of cylinder.}
    \label{fig:Cylinder_grasp}
\end{figure}
We conducted the experiment $20$ times, $10$ using the force estimation, and $10$ using the raw measurements, with the results shown in table~\ref{tab:experiments}.

When using the raw measurement, the system was not able to converge to the desired grasp and resulted in a force error of around $1N$, resulting in the tilting of the cylinder or loss of the grasp when lifting. Similar to the wrench example, we observe a higher success rate when using the estimated force values. For the failed experiments using the force estimation, in three cases the fingertip sensor had a slightly larger force error, the estimated forces were not able to fully converge to their desired values and when lifting, the object experienced a small tilt of $\approx 3^\circ$, but was held successfully. For the remaining failure case, the hysteresis of multiple taxels of the index finger created the illusion of a large force being sensed making the the index finger move away until it detaches. Once the hysteresis went down, the index finger was placed back onto the cylinder, but at that moment the cylinder had already moved considerably.\\

\subsubsection*{\textbf{Extrinsic Manipulation}}
We conducted two pivoting experiments using the contact between the object and the environment: 1) a 2-finger manipulation of a cube rotating about one of its edges as shown in Fig.~\ref{fig:Cube_hardware}; and 2) a 2-finger manipulation of a screwdriver pivoting about its tip, shown in Fig.~\ref{fig:screwdriver_hardware}. While we use a pinch grasp for both experiments, the interactions between the object and the environment are quite distinct. 
In both experiments, we use the tactile information to estimate the contact location on the object, then compute the force equilibrium considering the intrinsic contact (fingertips-object) and the extrinsic contact (object-table). The goal is for the objects to rotate about a pivot axis on the table. To accomplish this, using the distance between the pivot point and the contacts, the algorithm precomputes a trajectory for the base of the robotic hand to keep the object inside the fingers' workspace as we move the object to the desired pose. Given the initial angle of the object and the desired final angle, we plan the forces needed to accomplish the goal at each intrinsic contact point throughout the trajectory.

\begin{figure}
    \centering
    \includegraphics[width=1\linewidth]{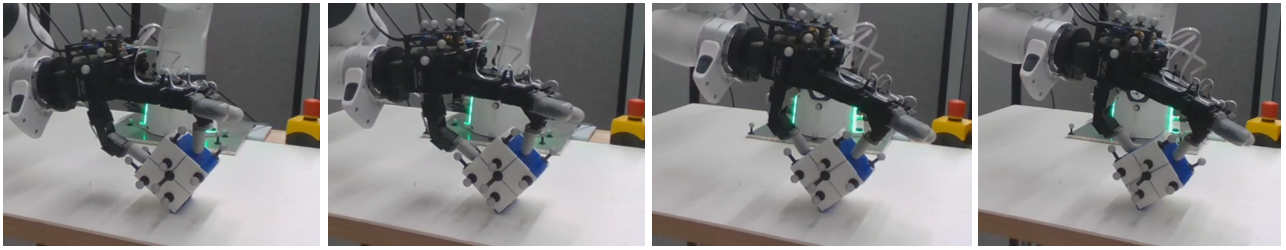}
    \caption{Cube turning hardware experiment.}
    \label{fig:Cube_hardware}
\end{figure}
\begin{figure}
    \centering
    \includegraphics[width=1\linewidth]{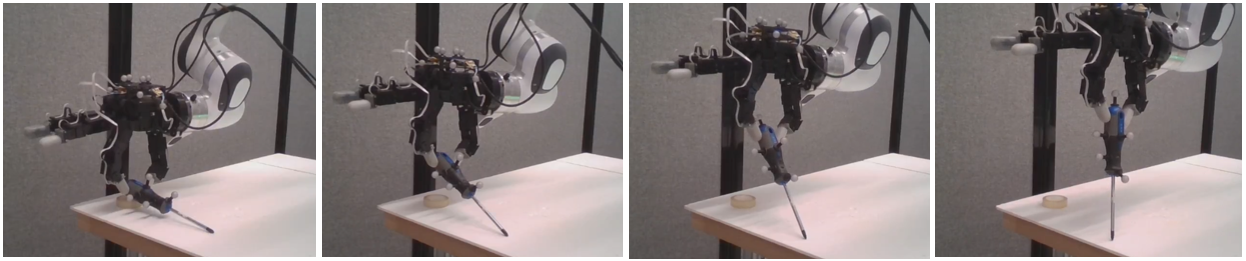}
    \caption{Screwdriver turning hardware experiment.}
    \label{fig:screwdriver_hardware}
\end{figure}

For the cube experiments, we rotate it about the yaw direction from the initial angle position of $58^\circ$ to $30^\circ$. We also tested rotations beyond $30^\circ$ but due to the friction between the cube and the table, the cube would start to slip on the table and contact would be lost, thus we keep the analysis up to the $30^\circ$ angle. We present $14$ experiments of the cube rotation, $7$ using the force estimation, and $7$ using the raw values. The mean trajectory for these experiments are shown in Fig.~\ref{fig:cube_results}. In all of these experiments, the hand follows a predefined trajectory about $2/3$ of the desired object's trajectory. This is done to show that the rotation is not only enabled by the arm moving, but the finger's controller also keeps track of the contact planned trajectories and forces. 

When using force estimation, the fingers converge to the planned desired forces and bring the cube closer to the desired angle. From the error plot in Fig.~\ref{fig:cube_results}, we can see that the angular error during the trajectory tracking is kept under $5^\circ$. For the experiments using the raw measurements, we can see that the overall error with respect to the trajectory is larger, with values between $5^\circ$ and $10^\circ$. Once the arm stops moving, the fingers are not able to move the object closer to the desired angle effectively.
\begin{figure}
    \centering
    \includegraphics[width=0.95\linewidth]{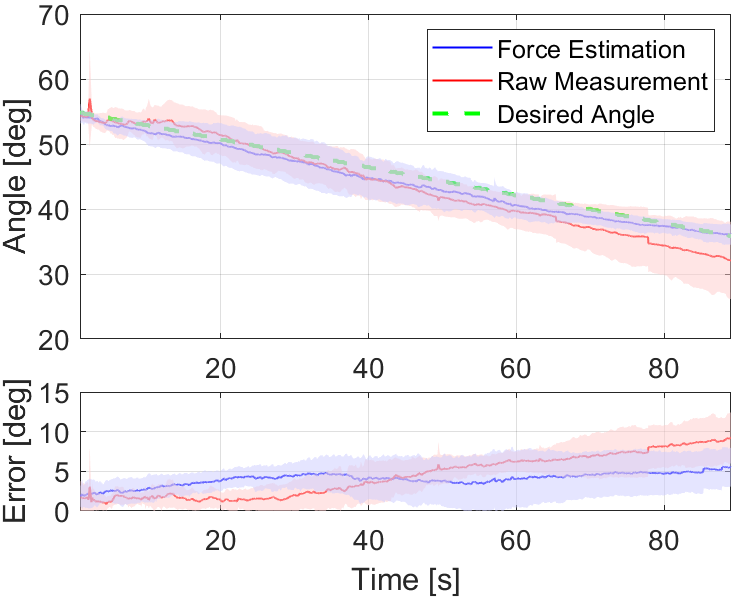}
    \caption{Mean (solid lines) and standard deviation (transparent areas) of yaw angles for cube turning experiments. Blue corresponds to the experiments using force estimation, while the red corresponds to experiments using raw measurements.}
    \label{fig:cube_results}
\end{figure}

For the screwdriver experiments, the screwdriver has a mass of $0.3kg$ and we have an approximate value for the friction coefficient of $0.6$. We tested $14$ instances of the screwdriver lifting experiment, with $7$ experiments using the force estimation and $7$ experiments using the raw force measurements. The mean trajectory for these experiments are shown in Fig.~\ref{fig:screwdriver_results}. When using the estimated forces, the system was successful in controlling the applied force while moving the hand to the final position. Fingers were able to keep the grasp, while the screwdriver's tip was always in contact with the table. For all the experiments using the estimated forces, the final angular error was under $7^\circ$. When using the raw measurement, the controller had trouble converging to the desired forces making the screwdriver rotate about its longitudinal axis. In two of these experiments, fingertips would end up lifting the screwdriver's tip from the table and eventually dropping the object as seen in the sudden drops of the angular position. For the other experiments where the screwdriver was not dropped, the grasp had shifted considerably, rotating the tool about its longitudinal axis, making further manipulation of the object very difficult.


\begin{figure}
    \centering
    \includegraphics[width=0.95\linewidth]{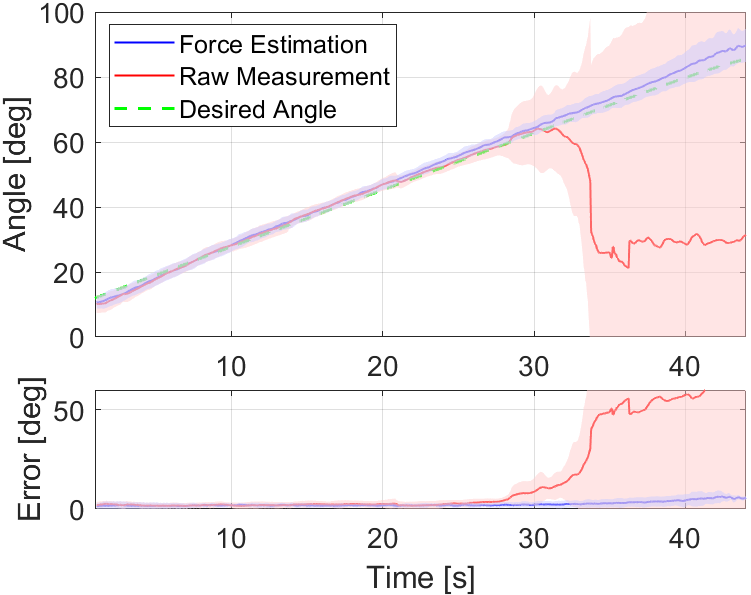}
    \caption{Mean (solid lines) and standard deviation (transparent areas) of angles for screwdriver turning experiments. Blue corresponds to the experiments using the force estimation, while the red correspond to experiments using raw measurements.}
    \label{fig:screwdriver_results}
\end{figure}

Thus, through these experiments, we showed that by using the force estimation to account for measurement noise and errors, the hand can perform manipulation of the object with high success, even when considering extrinsic contacts. 

\section{Discussion}
Despite the advancement of tactile sensors over the past years, their force-sensing capabilities are still in an early stage. However, assuming access to high-resolution tactile arrays with low thresholds of activation, and considering their ability to be placed on various surfaces such as robotic fingers, these sensors can be quite useful in dexterous manipulation applications. Thus, expanding the areas of application of these sensors and enabling force control in dexterous manipulation, has been our motivation for this work.

By presenting the characterization of the calibrated tactile sensors we have used, we intended to provide insight into the nature of measurements produced by these sensors. As seen in the tactile sensing section, while measured force is similar in nature to the ground truth value, the error is unpredictable as it depends on the contact location, applied pressure, direction of contact, and is different for individual taxels due to the slight differences in activation threshold, hysteresis, calibration, and manufacturing conditions. Thus, modeling and accounting for the error is extremely difficult, if not impossible. 

We have shown that controlling the system by relying on the raw force measurements is prone to failed grasps due to a lack of force equilibrium and overexerting of forces onto the object. While it might be possible to achieve successful grasps through compliance of the contacts or controller, it does not allow precise force control. In our work, we have shown that it is possible to use the raw measurements, to estimate an approximation of the true force that enables the robot to successfully control the applied force by a set of fingertips and achieve force equilibrium. As seen in the experimental results of grasping the wrench and the cylinder, we have considerably improved the success rate of grasping an object. The tactile sensor allows us to have a good contact location and direction with the object at each fingertip. We can use this contact location, along with the object parameters to compute the optimal force needed to grasp the object in force equilibrium, such that when lifted, the object will not slip out of the grasp or lose its desired orientation. 

The main reason for the inability of raw force measurements to accomplish equilibrium is the different error behaviors of the fingers which also change among interactions. Even when we start at force equilibrium (as done in the wrench experiment in Fig.~\ref{fig:wrench_grasp}, at the $40 s$ mark), if the raw measurements' errors are different, then the controller will have to apply different adjustments to each finger, thus breaking the force equilibrium and resulting in undesired object movement or loose grasp and lost contact. However, depending on the force error, there may be scenarios where the grasp can be accomplished, as seen in some of the raw measurement control experiments. In these cases, different fingers experienced similar errors, which allowed converges to the desired values. These grasps were probably over-exerting force but were still successful. Similarly, it is possible for the controller to fail when using the force estimations as seen in some of our experiments. For these failure cases, the main element at fault was the saturation of the tactile sensors of one or more fingertips. Due to the contact direction of the fingers, in some of the experiments, the raw measurement was saturated at a certain level despite the torque applied by the finger joint. While the other fingers might sense an increase in force which would be accounted for by the force estimator, the error in measurement increased beyond the modeled noise levels, leading to inaccurate behavior, overexerting forces, and breaking the force balance.  While we could increase the model's noise to account for these errors, this would also affect the value of the desired force, since the modeled error is higher, we have to increase the desired force to stay further away from the edges of the friction cone to accomplish a successful grasp. 

For the extrinsic force estimation, we have shown that our method can combine the noisy force measurements at contacts along with unknown forces at the contact between the object and the environment. In the cube experiment, we have accomplished force equilibrium considering the unknown environmental contact forces between the cube and the table. For the screwdriver, both fingers need to accomplish force balance such that they are able to lift the object but not over-exert force to avoid a torsional motion about the contact axis.

\section{Limitations and Future Work}
In this work, we assume that the contact location and contact normals are known and are fairly accurate. This assumption is not always true as different tactile elements might have different measurement errors, thus affecting the accuracy of computed center of pressure (contact location) and average normal, calculated based on weighted sum of individual readings. In our future work, we would like to consider that the contact location and direction have noise and need to be accounted for when computing the force equilibrium ellipsoid. Furthermore, tactile sensors' behaviors like impulse noise or interference were not modeled in our simulation. While the hardware experiments showed robustness to these elements, we would like to consider them in our future simulations.

Additionally, in this work we assumed quasi-static motions and static grasps where the object is in force equilibrium. In our future work, we would like to apply this to more dynamic motions as well as scenarios where applying an external wrench is desired, e.g. applying torque to turn a bolt by the wrench.

\section{Conclusion}
We have presented GeoDEx, a unified framework that can take inaccurate force measurements to plan, estimate, and control during dexterous manipulation. Given the object properties, such as weight, CoM, and friction, along with the tactile force measurements, our method computes the optimal force required to accomplish a grasp of an object in force equilibrium, then estimates the current applied force at each fingertip in contact, and uses this state estimation to control the system successfully. We have presented hardware experiments for two-, three-, and four-finger grasps with the hand at different configurations, showing a substantial improvement in grasp success rate compared to using direct force measurements provided by the calibrated tactile sensors. Our method can also take into account extrinsic contacts between the object and the environment and combine their effect when planning and estimating contact forces on the fingertips. We have shown two distinct examples of extrinsic manipulation using the table surface to manipulate a cube and a screwdriver successfully.









\bibliographystyle{unsrt}
\bibliography{biblio}



\end{document}